%% file: main.tex
\documentclass{article}
\usepackage{ijcai19}

\usepackage{times}
\usepackage{soul}
\usepackage{url}
\usepackage[hidelinks]{hyperref}
\usepackage[utf8]{inputenc}
\usepackage[small]{caption}
\usepackage{graphicx}
\usepackage{amsmath}
\usepackage{booktabs}
\usepackage{algorithm}
\usepackage{algorithmic}
\usepackage{color}
\usepackage{dsfont}
\usepackage{subfigure}
\usepackage{multirow}
\usepackage{amsfonts}
\usepackage{bbm}

\usepackage{appendix}
\usepackage{amsthm}
\newtheorem*{theorem}{Theorem}

\usepackage{fancyhdr}

\newcommand\blfootnote[1]{%
  \begingroup
  \renewcommand\thefootnote{}\footnote{#1}%
  \addtocounter{footnote}{-1}%
  \endgroup
}
\title{Using Natural Language for \\
Reward Shaping in Reinforcement Learning}

\author{Prasoon Goyal \and
Scott Niekum \and 
Raymond J. Mooney
\affiliations
Department of Computer Science \\
The University of Texas at Austin \\
\emails
\texttt{\{pgoyal, sniekum, mooney\}@cs.utexas.edu}
}

\hypersetup{draft}
\begin{document}

\maketitle
\begin{abstract}
Recent reinforcement learning (RL) approaches have shown strong performance in complex domains such as Atari games, but are often highly sample inefficient.
A common approach to reduce interaction time with the environment is to use reward shaping, which involves carefully designing reward functions that provide the agent intermediate rewards for progress towards the goal. 
However, designing appropriate shaping rewards is known to be difficult as well as time-consuming. 
In this work, we address this problem by using natural language instructions to perform reward shaping. 
We propose the LanguagE-Action Reward Network (LEARN), a framework that maps free-form natural language instructions to intermediate rewards based on actions taken by the agent. These intermediate language-based rewards can seamlessly be integrated into any standard reinforcement learning algorithm. 
We experiment with Montezuma's Revenge from the Atari Learning Environment, a popular benchmark in RL. 
Our experiments on a diverse set of 15 tasks demonstrate that, for the same number of interactions with the environment, language-based rewards lead to successful completion of the task 60\% more often on average, compared to learning without language. 

\end{abstract}

\blfootnote{Supplementary material link: \url{https://arxiv.org/abs/1903.02020}}
\input{intro}

\input{prelim}

\input{overview}

\input{learn}

\input{rl-with-lang}
\input{expts}

\input{related}
\input{concl}

\nocite{pytorchrl}

\section*{Acknowledgements}
This work was partially supported by an  NSF  NRI  grant  (IIS-1637736) and an Amazon Research Award.
\newpage
\bibliography{aaai}
\bibliographystyle{named}

\clearpage
\input{supplementary}

\end{document}

%% file: intro.tex
\section{Introduction}

Reinforcement learning (RL) has enjoyed much recent success in domains ranging from game-playing to real robotics tasks.
However, to make reinforcement learning useful for large-scale real-world applications, it is critical to be able to design reward functions that accurately and efficiently describe tasks.
For the sake of simplicity, a common strategy is to provide the agent with sparse rewards---for example, positive reward upon reaching a goal state, and zero reward otherwise.
However, it is well-known that learning is often difficult and slow in sparse reward settings \cite{vevcerik2017leveraging}.  By contrast, dense rewards can be easier to learn from, but are significantly more difficult to specify.
In this work, we address this issue by using natural language to provide dense rewards to RL agents in a manner that is easy to specify.

Consider the scenario in Figure~\ref{fig:motivation} from the Atari game Montezuma's Revenge. Suppose we want the agent to go to the left while jumping over the skull (as shown in the blue trajectory). If the agent is given a positive reward only when it reaches the end of the desired trajectory, it may need to spend a significant amount of time exploring the environment to learn that behavior.  Giving the agent intermediate rewards for progress towards the goal can help, a technique known as ``reward shaping" \cite{ng1999policy}. However, designing intermediate rewards is hard, particularly for non-experts.

\begin{figure}
\begin{center}
\includegraphics[width=0.60 \columnwidth]{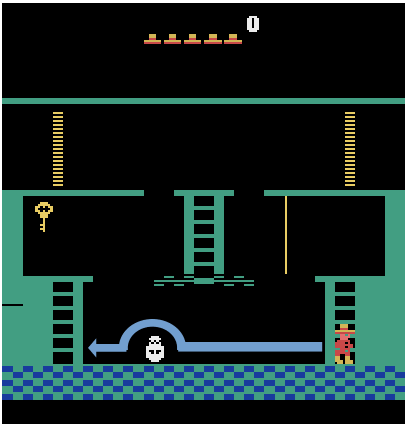}
\end{center}
\caption{An agent exploring randomly to complete the task described by the blue trajectory may need considerable amount of time to learn the behavior. By giving natural language instructions like ``Jump over the skull while going to the left'', we can give intermediate signals to the agent for faster learning.}
\label{fig:motivation}
\end{figure}

Instead, we propose giving the agent 
intermediate rewards using instructions in natural language. For instance, the agent can be given the following instruction:``Jump over the skull while going to the left'' to provide intermediate rewards that accelerate learning. 
Since natural language instructions can easily be provided even by non-experts, it will enable them to teach RL agents new skills more conveniently.

The main contribution of this work is a new framework
which takes arbitrary natural language instruction and the trajectory executed by the agent so far, and makes a prediction whether the agent is following the instruction, which can then be used as an intermediate reward. 
Our experiments show that by using such reward functions, we can speed up learning in sparse reward settings by guiding the exploration of the agent.

Using arbitrary natural language statements within reinforcement learning presents several challenges. First, a mapping between language and objects/actions must implicitly or explicitly be learned, a problem known as \emph{symbol grounding} \cite{harnad1990symbol}. For example, to make use of the instruction, ``Jump over the snake'', the system must be able to ground ``snake'' to appropriate  pixels in the current state (assuming the state is represented as an image) and ``jump'' to the appropriate action in the action space. Second, natural language instructions are often incomplete. For instance, it is possible that the agent is not directly next to the snake and must walk towards it before jumping. Third, natural language inherently involves ambiguity and variation. This could be due to different ways of referring to the objects/actions (e.g. ``jump'' vs. ``hop''), different amounts of information in the instructions (e.g. ``Jump over the snake" vs. ``Climb down the ladder after jumping over the snake"), or the level of abstraction at which the instructions are given (e.g. a high-level subgoal: ``Collect the key"
vs. low-level instructions: ``Jump over the obstacle. Climb up the ladder and jump to collect the key.") 

Once an instruction has been interpreted, we incorporate it into the RL system as an additional reward (as opposed to other options like defining a distribution over actions), since modifying the reward function allows using any standard RL algorithm for policy optimization.
We evaluate our approach on Montezuma's Revenge, a challenging game in the Atari domain \cite{bellemare2013arcade}, demonstrating that it effectively uses linguistic instructions to significantly speed learning, while also being robust to variation in instructions.

%% file: prelim.tex
\section{Overview of the Approach}
\label{sec:overview}

A Markov Decision Process (MDP) can be defined by the tuple $\langle S, A, T, R, \gamma \rangle$, where $S$ is a set of states, $A$ is a set of actions, $T : S \times A \times S \rightarrow [0,1]$ describes transition probabilities, $R : S \times A \rightarrow \mathds{R}$ is a reward function mapping the current state $s_{t}$ and current action $a_{t}$ to real-valued rewards, and $\gamma < 1$ is a discount factor. In this work, we consider an extension of the MDP framework, defined by  $\langle S, A, R, T, \gamma, l \rangle$, where $l \in L$ is a language command describing the intended behavior (with $L$ defined as the set of all possible language commands). We denote this language-augmented MDP as MDP+L. Given an MDP(+L), reinforcement learning can be used to learn an optimal policy $\pi^{*} : S \rightarrow A$ that maximizes expected sum of rewards.
We use $R_{ext}$ (``extrinsic") to denote the MDP reward function above, to avoid confusion with language-based rewards that we define in Section ~\ref{sec:rl-with-lang}.

%% file: overview.tex
\begin{figure}
\includegraphics[width=\columnwidth]{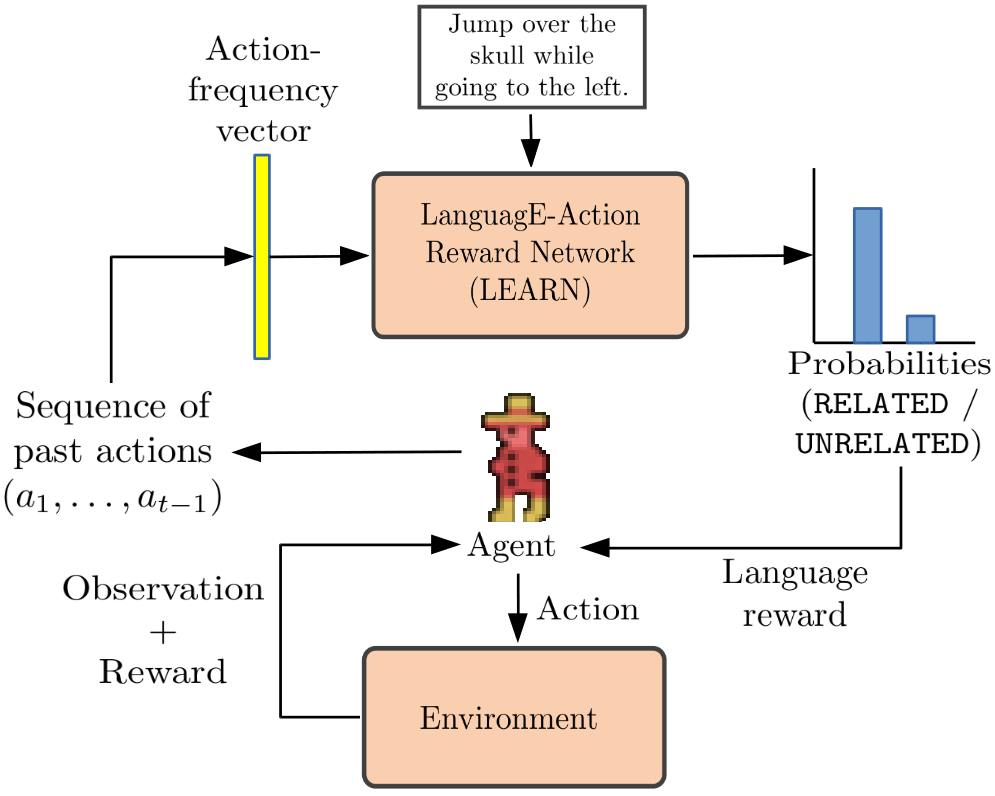}
\caption{Our framework consists of the standard RL module containing the agent-environment loop, augmented with a LanguagE-Action Reward Network (LEARN) module. 
}
\label{fig:schematic}
\end{figure}

In order to find an optimal policy in an MDP+L, we use a two-phase approach:

\paragraph{LanguagE-Action Reward Network (LEARN)} 
In this step, we train a neural network that takes paired (trajectory, language) data from the environment and predicts if the language describes the actions within the trajectory. To train the network, we collect natural language instructions for trajectories in the environment (Section~\ref{sec:lang-network}).

\paragraph{Language-aided RL} This step involves  using RL to learn a policy for the given MDP+L.
Given the trajectory executed by the agent so far and the language instruction, we use LEARN to predict whether the agent is making progress and use that prediction as a shaping reward (Section~\ref{sec:rl-with-lang}).
Note that since we are only modifying the reward function, this step is agnostic to the particular choice of RL algorithm.
A schematic diagram of the approach is given in Figure~\ref{fig:schematic}.

%% file: learn.tex
\section{LanguagE-Action Reward Network}
\label{sec:lang-network}

\subsection{Model}
\label{sec:learn-model}

LEARN takes in a trajectory and a language description and predicts whether the language describes the actions in the trajectory.
More formally, given a trajectory $\tau$, we create \emph{action-frequency vectors} from it as follows: \\
    1. Sample 
    two distinct timesteps $i$ and $j$ (such that $i < j$) from the set $\{1, \ldots, |\tau|\}$, where $|\tau|$ denotes the number of timesteps in $\tau$. Let $\tau[i:j]$ denote the segment of $\tau$ between timesteps $i$ and $j$. \\
    2. Create an \emph{action-frequency vector} $f$ from the actions in $\tau[i:j]$, where the dimensionality of $f$ is equal to the number of actions in the MDP+L, and the $k^{th}$ component of $f$ is the fraction of timesteps action $k$ appears in $\tau[i:j]$. \\

Using the above process, we create a dataset of $(f, l)$ pairs from a given set of $(\tau, l)$ pairs. 
Positive examples are created by sampling $f$ from a given trajectory $\tau$ and using the language description $l$ associated with $\tau$. Negative examples are created by (1) sampling an action-frequency vector $f$ from a given trajectory $\tau$, but choosing an alternate language description $l'$ sampled uniformly at random from the data excluding $l$,
or (2) creating a random action-frequency vector $f'$ and pairing it with the language description $l$.
These examples are used to train a neural network, as described below. 
Thus, given a pair $(f, l)$, the network learns to predict whether the action-frequency vector $f$ is related to the language description $l$ or not.

\paragraph{Neural network architecture} The action-frequency vector is passed through a sequence of fully-connected layers to get an encoded action vector with dimension $D_{1}$. Similarly, the natural language instruction is encoded into a vector with dimension $D_{2}$ as described below. The encoded action-frequency vector and language vector are then concatenated, and further passed through another sequence of fully-connected layers, each of dimension $D_3$, followed by a softmax layer. The final output of the network is a probability distribution over two classes -- {\tt RELATED} and {\tt UNRELATED}, corresponding to whether the action-frequency vector $f$ can be explained by the language instruction $l$.

\paragraph{Language encoder} To embed the natural language instruction, we experimented with three models:\\
(1) \textbf{InferSent} : In this model, we used a pretrained sentence embedding model \cite{conneau2017supervised}, which embeds sentences into a 4096-dimensional vector space. The 4096-dimensional vectors were projected to $D_2$-dimensional vectors using a fully-connected layer. We train only the projection layer during training, keeping the original sentence embedding model fixed.\\
(2) \textbf{GloVe+RNN} : In this model, we represent the sentence using pretrained 50-dimensional GloVe word embeddings \cite{pennington2014glove}, and train a two-layer GRU \cite{cho2014learning} encoder on top of it, while keeping the word embeddings fixed. We used the mean of the output vectors from the top layer as the encoding of the sentence. The hidden state size of the GRUs was set to $D_2$.\\
(3) \textbf{RNNOnly} : This model is identical to Glove+RNN, except instead of starting with pretrained GloVe vectors, we randomly initialize the word vectors and train both the word embeddings and the two-layer GRU encoder.

These three models trade-off prior domain knowledge with flexibility -- InferSent model starts with the knowledge of sentence similarity and is least flexible, GloVe+RNN model starts with word vectors and is more flexible in combining them to generate sentence embeddings, while RNNOnly starts with no linguistic knowledge and is completely flexible while learning word and sentence representations.

Our complete neural network architecture is shown in Figure~\ref{fig:nn-arch}. $D_1$, $D_2$ and $D_3$ were tuned using validation data.

\begin{figure}
\includegraphics[width=\columnwidth]{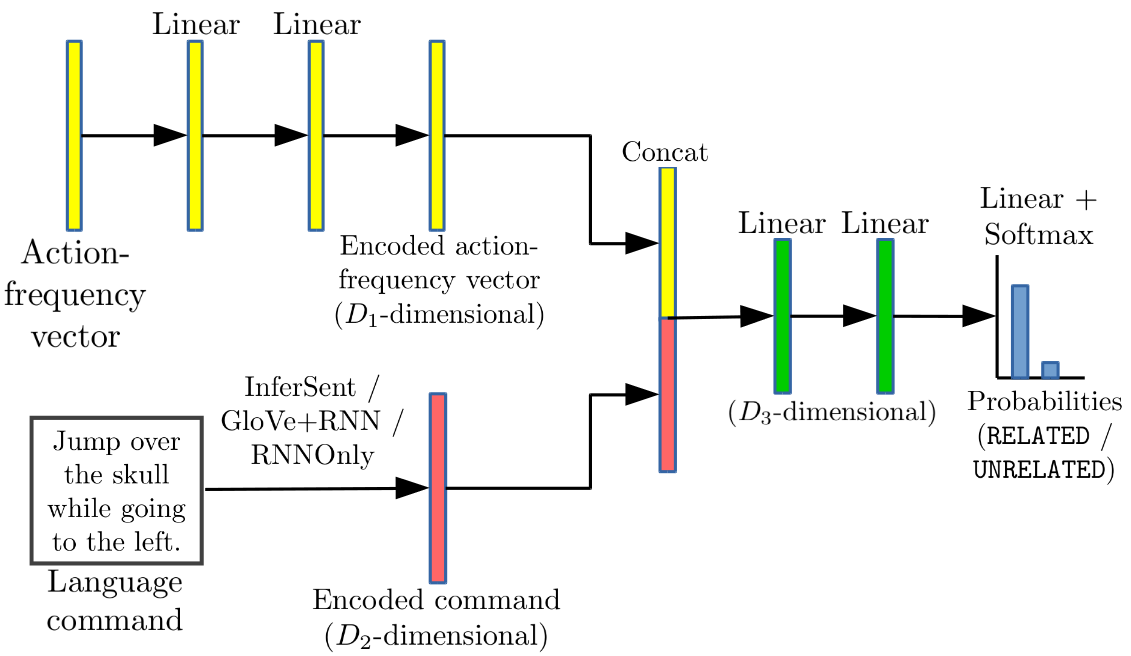}
\caption{Neural network architecture for LEARN (Section~\ref{sec:learn-model})
}
\label{fig:nn-arch}
\end{figure}

\paragraph{Training procedure} We used backpropagation with an Adam optimizer \cite{kingma2014adam} to train the above neural network for 50 epochs to minimize cross-entropy loss.

\input{data}

%% file: data.tex
\subsection{Data Collection}
\label{sec:data-collection}

To collect data for training LEARN, we generate trajectories in the environment, which may or may not be directly relevant for the final task(s). Then, for each trajectory, we get natural language annotations from human annotators, which are in the form of instructions that the agent should follow to go from the initial state of the trajectory to the final state. 

In our experiments, we used 20 trajectories from the Atari Grand Challenge dataset \cite{kurin2017atari}, which contains hundreds of crowd-sourced trajectories of human gameplays on 5 Atari games, including Montezuma's Revenge. The 20 trajectories we used contain a total of about 183,000 frames. From these trajectories, we extracted 2,708 equally-spaced clips (with overlapping frames), each three-seconds long.

To obtain language descriptions for these clips, we used Amazon Mechanical Turk.  Workers  were shown clips from the game and asked to provide corresponding language instructions.
Each annotator was asked to provide descriptions for 6 distinct clips, while each clip was annotated by 3 people.

To filter out bad annotations, we manually looked at each set of 6 annotations and discarded the set if any of them were generic statements (e.g. ``Good game!", ``Well played."), or if all the descriptions were very similar to one another (therefore suggesting that they are probably not related to the corresponding clips). After filtering, we obtained a total of 6,870 language descriptions. Note that the resulting dataset may still be quite noisy, since our filtering process doesn't explicitly check if the language instructions are related to the corresponding clips, nor do we correct for any spelling or grammatical errors. 

More details about the Amazon Mechanical Turk interface and example descriptions are included in the supplementary material.

%% file: rl-with-lang.tex
\section{Using Language-based Rewards in RL}
\label{sec:rl-with-lang}

To incorporate language information into RL, we use LEARN's predictions to generate intermediate rewards.
Given the sequence of actions $a_{1}, \ldots, a_{t-1}$ executed by the agent until timestep $t$ and the language instruction $l$ associated with the given MDP+L, we create an action-frequency vector $f_{t}$, by setting the $k^{th}$ component of $f$ equal to the fraction of timesteps action $k$ appears in $a_{1}, \ldots, a_{t-1}$. 
The resulting action-frequency vector $f$ and the language instruction $l$ are passed to LEARN.
Let the output probabilities corresponding to classes {\tt RELATED} and {\tt UNRELATED} be denoted as $p_R(f_{t})$ and $p_U(f_{t})$.
Note that since $l$ is fixed for a given MDP+L, $p_{R}(f_{t})$ and $p_{U}(f_{t})$ are functions of only the current action-frequency vector $f_{t}$.

Intuitively, trajectories that contain actions described by the language instruction more often will have higher values of $p_{R}(f_{t})$, compared to other trajectories. For instance, if the language instruction is ``Jump over the skull while going to the left", then trajectories with high frequencies corresponding to the ``jump" and ``left" actions will be considered more related to the language by LEARN.
Therefore, we can use these probabilities to define intermediate language-based rewards.
These intermediate rewards will enable the agent to explore more systematically, by choosing relevant actions more often than irrelevant actions.

To map the probabilities to language-based shaping rewards, we define a potential function for the current timestep as $\phi({f_{t}}) = p_{R}(f_{t}) - p_{U}(f_{t})$. The intermediate language-based reward is then defined as $R_{lang}(f_{t}) = \gamma \cdot \phi(f_{t}) - \phi(f_{t-1})$, where $\gamma$ is the discount factor for the MDP+L. We show in the supplementary material that a policy that is optimal under the original reward function ($R_{ext}$) is also optimal under the new reward function ($R_{ext} + R_{lang}$).

%% file: expts.tex
\section{Experimental Evaluation}

To validate the effectiveness of our approach, we conducted experiments on the Atari game Montezuma's Revenge. The game involves controlling an agent to navigate around multiple rooms. There are several types of objects within the rooms -- (1) ladders, ropes, doors, etc. that can be used to navigate within a room, (2) enemy objects (such as skulls and crabs) that the agent needs to escape from, (3) keys, daggers, etc. that can be collected. A screenshot from the game is included in Figure~\ref{fig:motivation}.
We selected this game because the rich set of objects and interactions allows for a wide variety of natural language descriptions.

The first step involved collecting (trajectory, language) pairs in the game as described in Section~\ref{sec:data-collection}.
The (trajectory, language) pairs were split into training and validation sets, such that there is no overlap between the frames in the training set and the validation set. In particular, Level 1 of Montezuma's revenge consists of 24 rooms, of which we use 14 for training, and the remaining 10 for validation and testing. The set of objects in both training and validation/test set are the same, but each room has only a subset of these objects arranged in different layouts.
We create a training dataset with 160,000 (action-frequency vector, language) pairs from the training set, and a validation dataset with 40,000 pairs from the validation set, which were used to train LEARN.

We define a set of 15 diverse tasks in multiple rooms, each of which requires the agent to go from a fixed start position to a fixed goal position while interacting with some of the objects present in the path.\footnote{Although the tasks (and corresponding descriptions) involve interactions with objects, we observe that just using actions, as we do in our approach, already gives improvements over the baseline, likely because most objects can be interacted with only in one way.}
For each task, the agent gets an extrinsic reward of +1 from the environment for reaching the goal, and an extrinsic reward of zero in all other cases. 

For each of the tasks, we generate a reference trajectory, and use Amazon Mechanical Turk to obtain 3 descriptions for the trajectory. We use each of these descriptions as language commands in our MDP+L experiments, as described below. Note that we do not use the reference trajectories to aid learning the policy in MDP+L; they are only used to collect language commands to be used in our experiments.

We use Proximal Policy Optimization, a popular on-policy RL algorithm \cite{schulman2017proximal}.
We train the policy for 500,000 timesteps for all our experiments.

\subsection{How much does language help?}

\paragraph{Settings} We experiment with 2 different RL setups to evaluate how much using language-based rewards help:\\
(1) \textbf{ExtOnly}: In this setup, we use the original environment reward, without using language-based reward. This is the standard MDP setup, and serves as our baseline.\\
(2) \textbf{Ext+Lang}: In this setup, in addition to the original environment reward that the agent gets on completing the task successfully, we also provide the agent potential-based language reward $R_{lang}$ at each step, as described in Section ~\ref{sec:rl-with-lang}.

\paragraph{Metrics} Performance is evaluated using two metrics:\\
(1) \textbf{AUC}: From each policy training run, we plot a graph with the number of timesteps on the x-axis and the number of successful episodes on the y-axis. The area under this curve is a measure of how quickly the agent learns, and is the metric we use to compare two policy training runs.\\
(2) \textbf{Final Policy}: To compare the final learned policy with ExtOnly and Ext+Lang, we perform policy evaluation at the end of 500,000 training steps. For each policy training run, we use the learned policy for an additional 10,000 timesteps without updating it, and record the number of successful episodes.

\paragraph{Hyperparameters} For the Ext+Lang setup, we perform validation over the three types of language encoders described in Section~\ref{sec:rl-with-lang} (InferSent / GloVe+RNN / RNNOnly). For each type of language encoder, we use the LEARN model with the best accuracy on the validation data. Further, we define the joint reward function as $R_{total} = R_{ext} + \lambda R_{lang}$. 
The type of language encoder and the hyperparameter $\lambda$ are selected using validation as described below.

\begin{figure}
\includegraphics[width=\columnwidth]{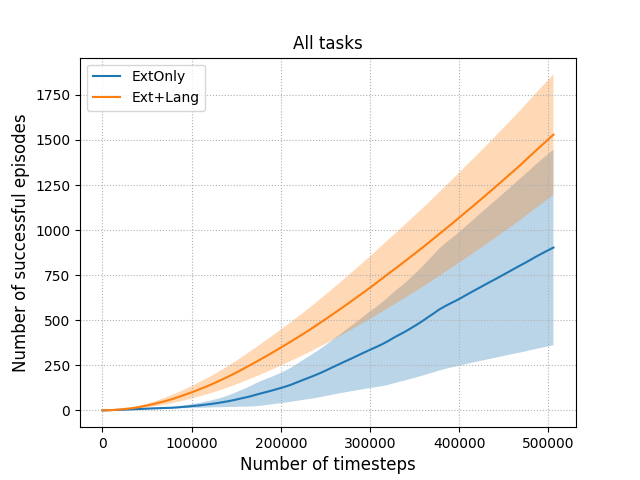}
\caption{Comparison of different reward functions: The solid lines represent the mean successful episodes averaged over all tasks, and the shaded regions represent 95\% confidence intervals.}
\label{fig:nolang-vs-lang-all}
\end{figure}

We treat each task as the test task in turn, using the remaining 14 tasks to find the best language encoder and $\lambda$. For each setting of the hyperparameters, we run policy training on all the validation tasks and each of the 3 descriptions, and compute AUC for each run. Since AUCs across tasks differ by orders of magnitude (due to varying task difficulties), we aggregate the scores across tasks as follows -- for each validation task, we compute a rank for each setting of the hyperparameters based on AUC, and then for each setting of the hyperparameters, we compute its average rank across the validation tasks. The setting with the best average rank is used for the test task.

\paragraph{Results} At test time, we performed 10 policy learning runs with different initializations for each task and each description. The results, averaged across all tasks and descriptions, are summarized in Figure~\ref{fig:nolang-vs-lang-all}, from which we can conclude that 
Ext+Lang learns much faster than ExtOnly, demonstrating that using natural language instructions for reward shaping is effective. In particular, the average number of successful episodes for ExtOnly after 500,000 timesteps is 903.12, while Ext+Lang achieves that score only after 358,464 timesteps, which amounts to a 30\% speed-up. Alternately, after 500,000 timesteps, Ext+Lang completes 1529.43 episodes on average, compared to 903.12 for ExtOnly, thereby giving a 60\% relative improvement.

\paragraph{Statistical Significance Tests} For each task, we perform an unpaired t-test between 10 runs of policy training with random initializations using ExtOnly reward function and 30 runs of policy training with random initializations using Ext+Lang reward function (3 descriptions $\times$ 10 runs per description), for both metrics. \\
(1) \textbf{AUC}: Of the total 15 tasks, Ext+Lang gives statistically significant improvement in 11 tasks, leads to statistically significant deterioration in 1 task, and makes no statistical difference in the remaining 3 tasks.
This agrees with the conclusions from Figure~\ref{fig:nolang-vs-lang-all}, that using language-based reward improves the efficiency of policy training on average. \\
(2) \textbf{Final Policy}: We observe that the number of successful episodes for the final policies is statistically significantly greater for Ext+Lang compared to ExtOnly in 8 out of 15 tasks, while the difference is not significant in the remaining 7 tasks. Further, averaged across all tasks, the number of successful episodes is more than twice with Ext+Lang compared to ExtOnly. These results suggests that using natural language for reward shaping often helps learn a better final policy, and rarely (if ever) results in a worse policy.

\subsection{Analysis of Language-based Rewards}
\label{sec:analysis}

\input{analysis.tex}

In order to analyze if the language-based rewards generated from LEARN actually correlate with language descriptions for the task, we compute the Spearman's rank correlation coefficient between each component of the action-frequency vector and corresponding prediction from LEARN over the 500,000 timesteps of policy training.
Correlation coefficients averaged across 10 runs of policy training for some selected tasks are reported in Table~\ref{tbl:analysis}.
Figure~\ref{fig:per-task} shows the policy training curves for these selected tasks.

This analysis supports some interesting observations:\\
    (1) For task 4 with simple descriptions, only the \texttt{DOWN} action is positively correlated with language-based reward. All other actions have a strong negative correlation with language-based reward, suggesting that the proposed approach discourages those actions, thereby aiding exploration.\\
    (2) For task 6 with more complex descriptions, LEARN correctly predicts language rewards to be correlated with actions \texttt{LEFT} and \texttt{DOWN}. For the third description, since the description does not instruct going down, the language reward is negatively correlated with the \texttt{DOWN} action. Indeed, we notice in our experiments that we obtain statistically significant improvement in AUC for the first two descriptions, while no statistically significant difference for the third description.\\
    (3) Task 14 represents a failure case. Language rewards predicted by LEARN are not well-correlated with the description, and consequently, using language-based rewards results in statistically significant deterioration in AUC. In general, we observe that groundings produced by LEARN for descriptions involving the word ``jump" are noisy. We hypothesize that this is because (i) the \texttt{JUMP} action typically appears with other actions like \texttt{LEFT} and \texttt{RIGHT}, and (ii) humans would typically use similar words to refer to \texttt{JUMP}, \texttt{JUMP-LEFT} and \texttt{JUMP-RIGHT} actions. These factors make it harder for the network to learn correct associations.

Note that LEARN does not see action names used in Table~\ref{tbl:analysis} (\texttt{NO-OP}, \texttt{JUMP}, etc.); instead, actions are represented as ordinals from 0 through 17. Thus, we see that our approach successfully learns to ground action names to actions in the environment.\footnote{While there are a total of 18 actions, we only report the most common 8 actions in Table~\ref{tbl:analysis} for brevity. The omitted 10 actions jointly constitute less that 1\% of the actions in the training data.}

\def \taskplotscale{0.31}
\begin{figure*}
\centering
\subfigure{
\includegraphics[scale=\taskplotscale]{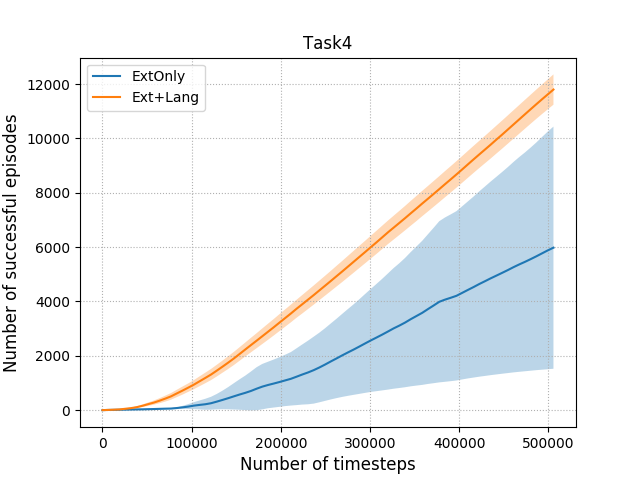}
}
\subfigure{
\includegraphics[scale=\taskplotscale]{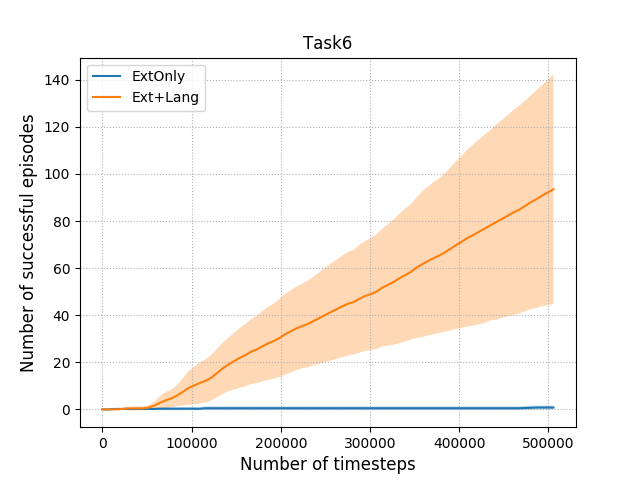}
}
\subfigure{
\includegraphics[scale=\taskplotscale]{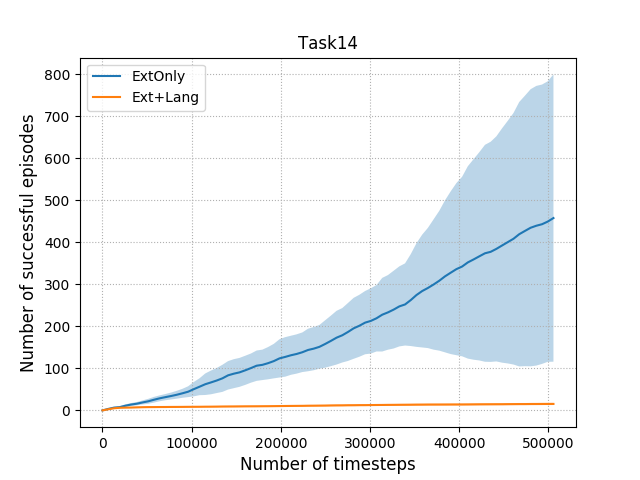}
} \\
\caption{Comparisons of different reward functions for selected tasks}
\label{fig:per-task}
\end{figure*}

%% file: analysis.tex
\begin{table*}
\footnotesize
\centering
\begin{tabular}{|c|l|r|r|r|r|r|r|r|r|}
\hline
\multicolumn{1}{|c|}{\multirow{2}{*}{\textbf{Task Id}}} & \multicolumn{1}{c|}{\multirow{2}{*}{\textbf{Description}}}                                                                         & \multicolumn{8}{c|}{\textbf{Correlation coefficients of different actions}}                                                                                                                                                                                                                                                                                                                            \\ \cline{3-10} 
\multicolumn{1}{|c|}{}                                  & \multicolumn{1}{c|}{}                                                                                                              & \multicolumn{1}{c|}{\texttt{NO-OP}} & \multicolumn{1}{c|}{\texttt{JUMP}} & \multicolumn{1}{c|}{\texttt{UP}} & \multicolumn{1}{c|}{\texttt{RIGHT}} & \multicolumn{1}{c|}{\texttt{LEFT}} & \multicolumn{1}{c|}{\texttt{DOWN}} & \multicolumn{1}{c|}{\texttt{\begin{tabular}[c]{@{}c@{}}JUMP-\\ RIGHT\end{tabular}}} & \multicolumn{1}{c|}{\texttt{\begin{tabular}[c]{@{}c@{}}JUMP-\\ LEFT\end{tabular}}} \\ \hline \hline
\multirow{3}{*}{4}                                      & climb down the ladder                                                                                                              & -0.60                               & -0.58                              & -0.59                            & -0.61                               & -0.55                              & 0.07                               & -0.57                                                                               & -0.56                                                                              \\ \cline{2-10} 
                                                        & go down the ladder to the bottom                                                                                                   & -0.58                               & -0.58                              & -0.58                            & -0.60                               & -0.53                              & 0.09                               & -0.59                                                                               & -0.60                                                                              \\ \cline{2-10} 
                                                        & move on spider and down on the lader                                                                                               & -0.58                               & -0.54                              & -0.59                            & -0.60                               & -0.49                              & 0.10                               & -0.58                                                                               & -0.56                                                                              \\ \hline \hline
\multirow{3}{*}{6}                                      & \begin{tabular}[c]{@{}l@{}}{\scriptsize go to the left and go under skulls and then down the ladder}\end{tabular}                              & -0.37                               & -0.40                              & -0.49                            & -0.43                               & 0.33                               & 0.16                               & -0.46                                                                               & -0.01                                                                              \\ \cline{2-10} 
                                                        & go to the left and then go down the ladder                                                                                         & -0.24                               & -0.26                              & -0.35                            & -0.31                               & 0.28                               & 0.36                               & -0.34                                                                               & -0.04                                                                              \\ \cline{2-10} 
                                                        & move to the left and go under the skulls                                                                                           & -0.16                               & -0.25                              & -0.60                            & -0.48                               & 0.27                               & -0.63                              & -0.52                                                                               & -0.40                                                                               \\ \hline \hline
\multirow{3}{*}{14}                                     & Jump once then down                                                                                                                & 0.00                                & 0.07                               & -0.15                            & -0.13                               & 0.51                               & 0.50                               & 0.09                                                                                & 0.52                                                                               \\ \cline{2-10} 
                                                        & go down the rope and to the bottom                                                                                                 & -0.03                               & 0.10                               & -0.16                            & 0.56                                & 0.54                               & 0.33                               & 0.28                                                                                & 0.01                                                                               \\ \cline{2-10} 
                                                        & jump once and climb down the stick                                                                                                 & 0.11                                & 0.11                               & 0.06                             & 0.04                                & 0.14                               & 0.40                               & 0.25                                                                                & 0.11                                                                               \\ \hline
\end{tabular}
\caption{Analysis of language-based rewards}
\label{tbl:analysis}
\end{table*}

%% file: related.tex
\section{Related Work}


Prior work on combining RL and natural language can be divided into two classes. The first class uses reinforcement learning to solve NLP tasks, such as summarization \cite{paulus2017deep}, question-answering \cite{xiong2017dcn+} and dialog generation \cite{li2016deep}. The second class, in which our approach lies, uses natural language to aid RL.

Regarding methods that use NLP to help RL, some recent approaches map natural language to a reward function. \cite{williams2017learning} and \cite{arumugam2017accurately} map language to a reward function in an object-oriented MDP framework. However, these approaches use a predefined set of objects, object properties and spatial relations, and/or use simple language-based features, which is difficult to scale to more complex environments and instructions. Our approach, on the other hand, learns to ground natural language concepts to actions directly from data.

\cite{misra2017mapping} use natural language to describe the goal, which is combined with the state information to learn a policy in contextual bandit setting. However, they use distance from the goal and from reference trajectories for reward shaping.
\cite{kuhlmann2004guiding} map natural language to a set of rules which are then used to increase or decrease the probability of choosing an action during reinforcement learning. Extending this to complex environments would require engineering how each rule affects the probabilities of different actions.
Our approach, on the other hand, uses the natural language instruction itself for reward shaping, directly generating rewards from language, thereby reducing human effort.


\cite{branavan2012learning} extract features from natural language instructions, and incorporate them into the action-value function. More recently, \cite{bahdanau2018learning} proposed an adversarial learning framework wherein a discriminator distinguishes between a fixed set of good (instruction, state) pairs and (instruction, state) pairs generated by the current policy, and this is used as a reward function to simultaneously improve the policy.
A key difference between these approaches and our approach is that they learn linguistic features jointly during reinforcement learning, while we learn to map language to a reward function offline, which could be beneficial if interaction with the environment is expensive. However, our approach requires pairs of trajectories and natural language instructions for offline training.

\cite{branavan2012learning2} and \cite{kaplan2017beating} use natural language to do high-level planning.
These approaches are orthogonal to our work, in that these approaches can be used to generate subgoals at a high-level, whereas our approach can be used to make exploration faster at a lower-level.

Finally, our model is related to that in \cite{wang2018rcm-sil}, which also uses intermediate language-based rewards in RL. However, their goal is to use RL to improve natural language instruction-following, while our focus is on the reverse problem of using instructions to improve RL performance. 

%% file: concl.tex
\section{Conclusions and Future Work}

We propose LanguagE Action Reward Network (LEARN), a framework trained on paired (trajectory, language) data in an environment to predict if the actions in a trajectory match the language description. 
The outputs of the network are used to generate intermediate rewards for reinforcement learning. 
We show in our experiments that these language-based rewards can be used to train faster and learn a better policy for sparse reward settings.
Further, since the modality by which information is given to the agent is natural language, this approach can potentially be used even by non-experts to specify tasks to RL agents.

While our approach achieves promising improvements over the baseline, there are several possible extensions of the approach:\\
1) {\bf Temporal ordering:} Our approach aggregates the sequence of past actions into an action-frequency vector, thereby losing temporal information. Therefore a possible extension is to look at the complete action sequences. \\
2) {\bf State-based rewards:} Currently, the language-based reward is a function of only the past actions. As such, the model cannot utilize natural language descriptions that refer to objects in the state (e.g. ``Go towards the \emph{ladder}", ``avoid the \emph{skulls}".) Modelling the language-based reward as a function of both the past states and actions should allow the agent to benefit from such language descriptions. \\
3) {\bf Multi-step instructions:} The current approach only handles a single instruction. One way to handle multiple instructions is to have another module (trained / heuristic-based) to predict if a language instruction has been completed or not. This could then be used in conjunction with our current approach, where the agent starts following the first instruction, and transitions to the next one when this new module predicts that the current instruction has been completed.

%% file: supplementary.tex
\appendix

\input{example-annot}
\input{proof}

\input{sensitivity}
\input{amt-interface}

%% file: example-annot.tex
\section{Example Annotations}

Table~\ref{tbl:annotation-examples} shows 20 randomly selected annotations collected using Amazon Mechanical Turk (after the filtering process described in Section~\ref{sec:data-collection}). Note that the annotations have a significant amount of variation, both in terms of length and vocabulary. Further, several descriptions (1) contain spelling errors (e.g. ``climbling" in annotation 6 and ``dwon" in annotation 7), (2) are ill-formed (e.g. annotation 2) or (3) are not very informative (e.g. annotations 1 and 7). We do not filter out or correct these annotations, as the process requires significant manual effort. Thus, our method is able to extract useful information from these annotations even in the presence of noise.

\begin{table}[h]
\small
\centering
\begin{tabular}{|ll|}
\hline
1.  & wait                                                                                                                                  \\ \hline
2.  & using the ladder on standing                                                                                                          \\ \hline
3.  & going slow and climb down the ladder                                                                                                  \\ \hline
4.  & move down the ladder and walk left                                                                                                    \\ \hline
5.  & go left watch the trap and move on                                                                                                    \\ \hline
6.  & climbling down the ladder                                                                                                             \\ \hline
7.  & ladder dwon and running this away                                                                                                     \\ \hline
8.  & stay in place on the ladder.                                                                                                          \\ \hline
9.  & go down the ladder                                                                                                                    \\ \hline
10. & go right and climb up the ladder                                                                                                      \\ \hline
11. & just jump and little move to right side                                                                                               \\ \hline
12. & run all the way to the left.                                                                                                          \\ \hline
13. & go left jumping once                                                                                                                  \\ \hline
14. & go left                                                                                                                               \\ \hline
15. & \begin{tabular}[c]{@{}l@{}}move right and jump over green\\ creature then go down the ladder\end{tabular}                             \\ \hline
16. & hop over to the middle ledge                                                                                                          \\ \hline
17. & \begin{tabular}[c]{@{}l@{}}wait for the two skulls and dodge\\ them in the middle\end{tabular}                                        \\ \hline
18. & walk to the left and then jump down                                                                                                   \\ \hline
19. & jump to collected gold coin and little move                                                                                           \\ \hline
20. & \begin{tabular}[c]{@{}l@{}}wait for the platform to materialize then\\ walk and leap to your right to collect the coins.\end{tabular} \\ \hline
\end{tabular}
\caption{Examples of descriptions collected using Amazon Mechanical Turk}
\label{tbl:annotation-examples}
\end{table}

%% file: proof.tex
\section{Policy Invariance}

In this section, we show that using action-frequency vectors for reward shaping does not change the optimal policy.


\begin{theorem}
Let $M = \langle S, A, T, R, \gamma \rangle$ be a given MDP, and $R_{lang}(f_{t}) = \gamma \cdot \phi(f_{t}) - \phi(f_{t-1})$ be a shaping reward function, where $f_{t}$ is the action-frequency vector corresponding to actions $a_{1}, \ldots, a_{t}$ as defined in Section~\ref{sec:learn-model}, and $\phi$ be a potential function. Then, an optimal policy in $M$ is also an optimal policy in the MDP $M' = \langle S, A, T, R+F, \gamma \rangle$.
\end{theorem}

\begin{proof}

Define an MDP $\widehat{M} = \langle \widehat{S}, \widehat{A}, \widehat{T}, \widehat{R}, \gamma \rangle$, such that
\begin{itemize}
    \item For all $s \in S$ and $g \in \mathbb{Z}_{+}^{|A|}$, $(s, g) \in \widehat{S}$. \\
    ($g$ is the vector of counts of each action.)
    \item $\widehat{A} = A$.
    \item $\widehat{R}((s, g), a, (s', g')) = R(s, a, s') \mathbbm{1}[g, a, g' \mbox{consistent}]$ \\
    (Consistent refers to whether $g'$ is obtained from $g$ on taking action $a$.)
    \item $\widehat{T}((s, g), a, (s', g')) = T(s, a, s') \mathbbm{1}[g, a, g' \mbox{consistent}]$
\end{itemize}

\bigskip
\noindent
Let $Q_{M}^{*}$ be the optimal Q-function for the original MDP $M$. Define \\
\[
\widehat{Q}_{M'}((s, g), a) = Q_{M}^{*}(s, a)
\]

\noindent
Now,
\begin{equation}
\small
\label{eqn:q-opt}
\begin{split}
&\mathbb{E}_{(s', g') \sim \widehat{T}} [\widehat{R}((s, g), a, (s', g')) + \gamma \max_{a'}\widehat{Q}_{M'}((s', g'), a')] \\
&= \mathbb{E}_{(s', g') \sim \widehat{T}} [R(s, a, s') \mathbbm{1}[f, a, g' \mbox{consistent}] + \gamma \max_{a'}Q_{M}^{*}(s', a')] \\
&= \mathbb{E}_{s' \sim T} [R(s, a, s') + \gamma \max_{a'}Q_{M}^{*}(s', a')] \\
&= Q_{M}^{*}(s, a) \\
&= \widehat{Q}_{M'}((s, g), a)
\end{split}
\end{equation}

\noindent
The second step involves expanding out the expectation w.r.t. $\widehat{T}$, removing the inconsistent terms, since they get multiplied by zero, and converting back to expectation w.r.t. $T$.

\noindent
Thus, $\widehat{Q}_{M'}$ satisfies the Bellman optimality equation for $M'$. \\

\bigskip
\noindent
Next, let $\pi_{M}^{*}$ be an optimal policy for $M$. Then,
\begin{equation*}
\pi_{M}^{*}(s) \in \arg\max_{a} Q_{M}^{*}(s, a)
\end{equation*}

\noindent
Defining $\widehat{\pi}_{M'}((s, g)) = \pi_{M}^{*}(s)$, we get
\begin{equation}
\small
\label{eqn:pi-opt}
\begin{split}
\pi_{M'}((s, g)) 
&= \pi_{M}^{*}(s) \\
&\in \arg\max_{a} Q_{M}^{*}(s, a) \\
&= \arg\max_{a} \widehat{Q}_{M'}((s, f), a)
\end{split}
\end{equation}

\noindent
Using equations \ref{eqn:q-opt} and \ref{eqn:pi-opt}, we can conclude that $\widehat{\pi}_{M'}((s, g))$ is optimal in $M'$. \\

Note that $M'$ could admit other optimal policies as well, which could potentially also depend on $g$. 

Since the states in $M'$ contain the history of action counts, our proposed potential-based shaping reward can now be defined as a function of only the state in $M'$. From \cite{ng1999policy}, these shaping rewards do not change the optimal policy.
\end{proof}

%% file: sensitivity.tex
\section{Sensitivity Analysis}

To better understand the relation between the LEARN module and RL, we added varying amounts of noise to the output of LEARN. Specifically, Gaussian noise $\mathcal{N}(0, \sigma)$ was added to the potential function as described in Section~\ref{sec:rl-with-lang}, where $\sigma$ was varied from 0.01 to 1.0. The results for Task 8 are shown in Figure~\ref{fig:sensitivity}, from which we can see that the language-based rewards improve over the baseline even with significant amounts of noise. This suggests that the predictions from the LEARN module are fairly robust.

\begin{figure}
\includegraphics[width=\columnwidth]{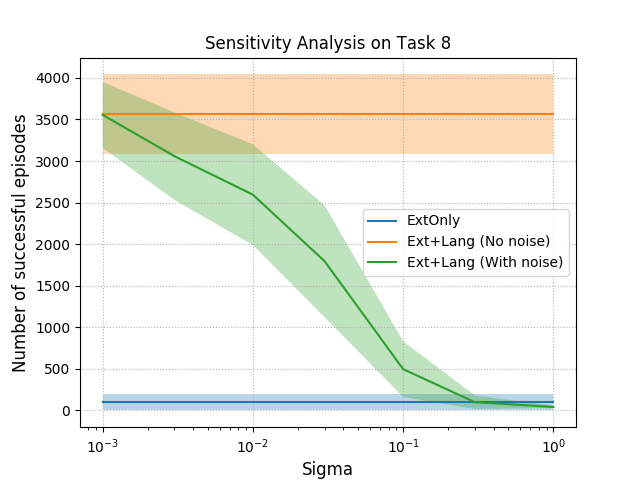}
\caption{Effect of adding noise to the predictions of LEARN: The solid lines represent the mean successful episodes averaged over all tasks, and the shaded regions represent 95\% confidence intervals.}
\label{fig:sensitivity}
\end{figure}

%% file: amt-interface.tex
\section{Amazon Mechanical Turk interface}

Figure~\ref{fig:amt-interface} shows the interface used on Amazon Mechanical Turk for data collection.

\begin{figure*}
\centering
\subfigure{
\includegraphics[width=0.6\textwidth]{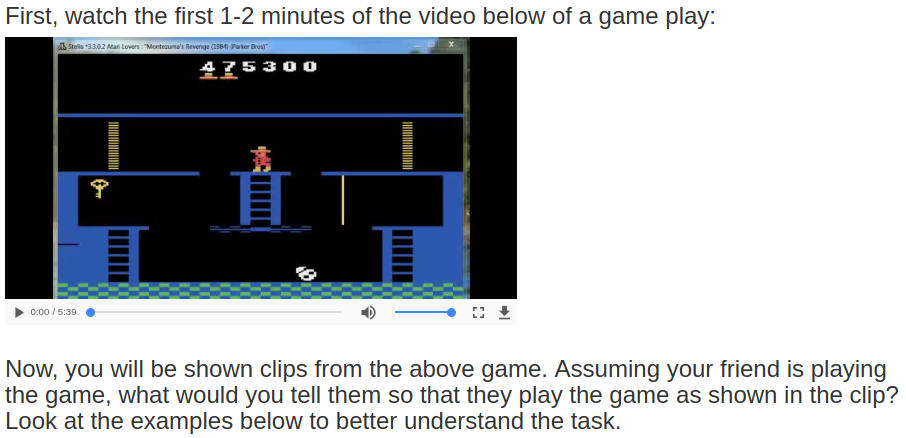}
}\\
\subfigure{
\includegraphics[width=0.6\textwidth]{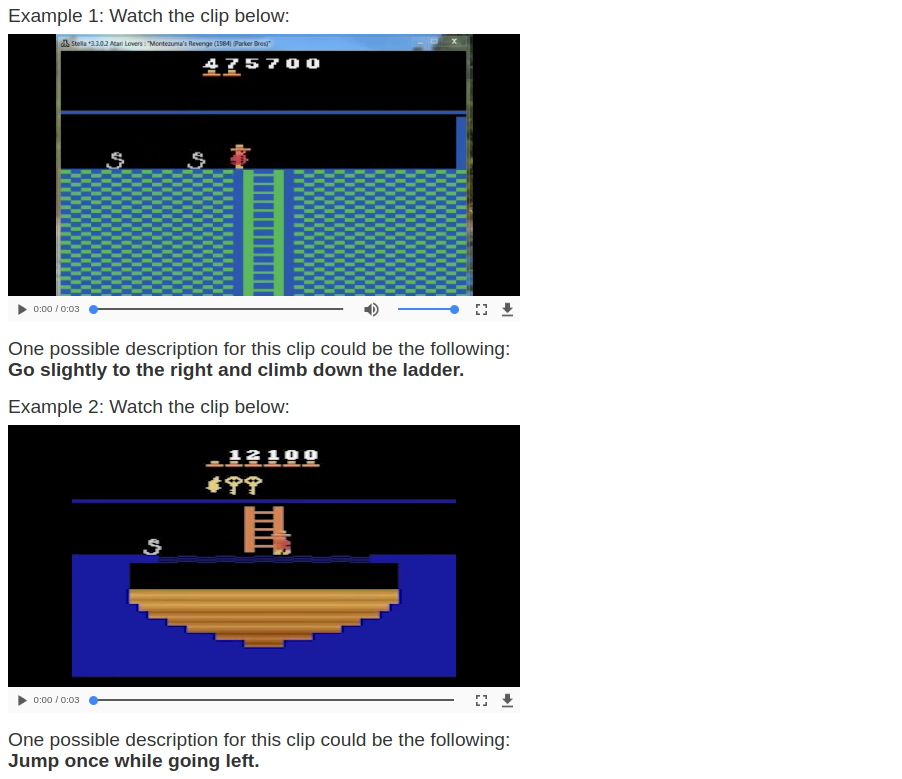}
}\\
\subfigure{
\includegraphics[width=0.6\textwidth]{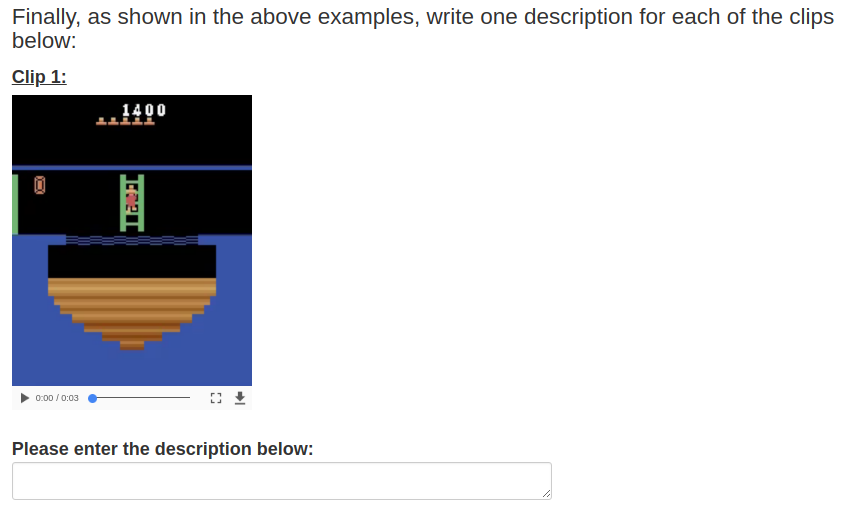}
}
\caption{Sample Mechanical Turk HIT for collecting natural language descriptions.}
\label{fig:amt-interface}
\end{figure*}

%% file: main.bbl
\begin{thebibliography}{}

\bibitem[\protect\citeauthoryear{Arumugam \bgroup \em et al.\egroup
  }{2017}]{arumugam2017accurately}
Dilip Arumugam, Siddharth Karamcheti, Nakul Gopalan, Lawson~LS Wong, and
  Stefanie Tellex.
\newblock Accurately and efficiently interpreting human-robot instructions of
  varying granularities.
\newblock {\em arXiv preprint arXiv:1704.06616}, 2017.

\bibitem[\protect\citeauthoryear{Bahdanau \bgroup \em et al.\egroup
  }{2018}]{bahdanau2018learning}
Dzmitry Bahdanau, Felix Hill, Jan Leike, Edward Hughes, Pushmeet Kohli, and
  Edward Grefenstette.
\newblock Learning to follow language instructions with adversarial reward
  induction.
\newblock {\em arXiv preprint arXiv:1806.01946}, 2018.

\bibitem[\protect\citeauthoryear{Bellemare \bgroup \em et al.\egroup
  }{2013}]{bellemare2013arcade}
Marc~G Bellemare, Yavar Naddaf, Joel Veness, and Michael Bowling.
\newblock The arcade learning environment: An evaluation platform for general
  agents.
\newblock {\em Journal of Artificial Intelligence Research}, 47:253--279, 2013.

\bibitem[\protect\citeauthoryear{Branavan \bgroup \em et al.\egroup
  }{2012a}]{branavan2012learning2}
SRK Branavan, Nate Kushman, Tao Lei, and Regina Barzilay.
\newblock Learning high-level planning from text.
\newblock In {\em Proceedings of the 50th Annual Meeting of the Association for
  Computational Linguistics: Long Papers-Volume 1}, pages 126--135. Association
  for Computational Linguistics, 2012.

\bibitem[\protect\citeauthoryear{Branavan \bgroup \em et al.\egroup
  }{2012b}]{branavan2012learning}
SRK Branavan, David Silver, and Regina Barzilay.
\newblock Learning to win by reading manuals in a monte-carlo framework.
\newblock {\em Journal of Artificial Intelligence Research}, 43:661--704, 2012.

\bibitem[\protect\citeauthoryear{Cho \bgroup \em et al.\egroup
  }{2014}]{cho2014learning}
Kyunghyun Cho, Bart Van~Merri{\"e}nboer, Caglar Gulcehre, Dzmitry Bahdanau,
  Fethi Bougares, Holger Schwenk, and Yoshua Bengio.
\newblock Learning phrase representations using rnn encoder-decoder for
  statistical machine translation.
\newblock {\em arXiv preprint arXiv:1406.1078}, 2014.

\bibitem[\protect\citeauthoryear{Conneau \bgroup \em et al.\egroup
  }{2017}]{conneau2017supervised}
Alexis Conneau, Douwe Kiela, Holger Schwenk, Lo\"{i}c Barrault, and Antoine
  Bordes.
\newblock Supervised learning of universal sentence representations from
  natural language inference data.
\newblock In {\em Proceedings of the 2017 Conference on Empirical Methods in
  Natural Language Processing}, pages 670--680, Copenhagen, Denmark, September
  2017. Association for Computational Linguistics.

\bibitem[\protect\citeauthoryear{Harnad}{1990}]{harnad1990symbol}
Stevan Harnad.
\newblock The symbol grounding problem.
\newblock {\em Physica D: Nonlinear Phenomena}, 42(1-3):335--346, 1990.

\bibitem[\protect\citeauthoryear{Kaplan \bgroup \em et al.\egroup
  }{2017}]{kaplan2017beating}
Russell Kaplan, Christopher Sauer, and Alexander Sosa.
\newblock Beating atari with natural language guided reinforcement learning.
\newblock {\em arXiv preprint arXiv:1704.05539}, 2017.

\bibitem[\protect\citeauthoryear{Kingma and Ba}{2014}]{kingma2014adam}
Diederik~P Kingma and Jimmy Ba.
\newblock Adam: A method for stochastic optimization.
\newblock {\em arXiv preprint arXiv:1412.6980}, 2014.

\bibitem[\protect\citeauthoryear{Kostrikov}{2018}]{pytorchrl}
Ilya Kostrikov.
\newblock Pytorch implementations of reinforcement learning algorithms.
\newblock \url{https://github.com/ikostrikov/pytorch-a2c-ppo-acktr}, 2018.

\bibitem[\protect\citeauthoryear{Kuhlmann \bgroup \em et al.\egroup
  }{2004}]{kuhlmann2004guiding}
Gregory Kuhlmann, Peter Stone, Raymond Mooney, and Jude Shavlik.
\newblock Guiding a reinforcement learner with natural language advice: Initial
  results in robocup soccer.
\newblock In {\em The AAAI-2004 workshop on supervisory control of learning and
  adaptive systems}. San Jose, CA, 2004.

\bibitem[\protect\citeauthoryear{Kurin \bgroup \em et al.\egroup
  }{2017}]{kurin2017atari}
Vitaly Kurin, Sebastian Nowozin, Katja Hofmann, Lucas Beyer, and Bastian Leibe.
\newblock The atari grand challenge dataset.
\newblock {\em arXiv preprint arXiv:1705.10998}, 2017.

\bibitem[\protect\citeauthoryear{Li \bgroup \em et al.\egroup
  }{2016}]{li2016deep}
Jiwei Li, Will Monroe, Alan Ritter, Michel Galley, Jianfeng Gao, and Dan
  Jurafsky.
\newblock Deep reinforcement learning for dialogue generation.
\newblock {\em arXiv preprint arXiv:1606.01541}, 2016.

\bibitem[\protect\citeauthoryear{Misra \bgroup \em et al.\egroup
  }{2017}]{misra2017mapping}
Dipendra Misra, John Langford, and Yoav Artzi.
\newblock Mapping instructions and visual observations to actions with
  reinforcement learning.
\newblock {\em arXiv preprint arXiv:1704.08795}, 2017.

\bibitem[\protect\citeauthoryear{Ng \bgroup \em et al.\egroup
  }{1999}]{ng1999policy}
Andrew~Y Ng, Daishi Harada, and Stuart Russell.
\newblock Policy invariance under reward transformations: Theory and
  application to reward shaping.
\newblock In {\em ICML}, volume~99, pages 278--287, 1999.

\bibitem[\protect\citeauthoryear{Paulus \bgroup \em et al.\egroup
  }{2017}]{paulus2017deep}
Romain Paulus, Caiming Xiong, and Richard Socher.
\newblock A deep reinforced model for abstractive summarization.
\newblock {\em arXiv preprint arXiv:1705.04304}, 2017.

\bibitem[\protect\citeauthoryear{Pennington \bgroup \em et al.\egroup
  }{2014}]{pennington2014glove}
Jeffrey Pennington, Richard Socher, and Christopher Manning.
\newblock Glove: Global vectors for word representation.
\newblock In {\em Proceedings of the 2014 conference on empirical methods in
  natural language processing (EMNLP)}, pages 1532--1543, 2014.

\bibitem[\protect\citeauthoryear{Schulman \bgroup \em et al.\egroup
  }{2017}]{schulman2017proximal}
John Schulman, Filip Wolski, Prafulla Dhariwal, Alec Radford, and Oleg Klimov.
\newblock Proximal policy optimization algorithms.
\newblock {\em arXiv preprint arXiv:1707.06347}, 2017.

\bibitem[\protect\citeauthoryear{Ve{\v{c}}er{\'\i}k \bgroup \em et al.\egroup
  }{2017}]{vevcerik2017leveraging}
Matej Ve{\v{c}}er{\'\i}k, Todd Hester, Jonathan Scholz, Fumin Wang, Olivier
  Pietquin, Bilal Piot, Nicolas Heess, Thomas Roth{\"o}rl, Thomas Lampe, and
  Martin Riedmiller.
\newblock Leveraging demonstrations for deep reinforcement learning on robotics
  problems with sparse rewards.
\newblock {\em arXiv preprint arXiv:1707.08817}, 2017.

\bibitem[\protect\citeauthoryear{Wang \bgroup \em et al.\egroup
  }{2018}]{wang2018rcm-sil}
Xin Wang, Qiuyuan Huang, Asli Celikyilmaz, Jianfeng Gao, Dinghan Shen,
  Yuan-Fang Wang, William~Yang Wang, and Lei Zhang.
\newblock Reinforced cross-modal matching and self-supervised imitation
  learning for vision-language navigation.
\newblock {\em arXiv preprint arXiv:1811.10092}, 2018.

\bibitem[\protect\citeauthoryear{Williams \bgroup \em et al.\egroup
  }{2017}]{williams2017learning}
Edward~C Williams, Mina Rhee, Nakul Gopalan, and Stefanie Tellex.
\newblock Learning to parse natural language to grounded reward functions with
  weak supervision.
\newblock In {\em AAAI Fall Symposium on Natural Communication for Human-Robot
  Collaboration}, 2017.

\bibitem[\protect\citeauthoryear{Xiong \bgroup \em et al.\egroup
  }{2017}]{xiong2017dcn+}
Caiming Xiong, Victor Zhong, and Richard Socher.
\newblock Dcn+: Mixed objective and deep residual coattention for question
  answering.
\newblock {\em arXiv preprint arXiv:1711.00106}, 2017.

\end{thebibliography}
